\relax
\documentclass[letterpaper]{article} 
\usepackage{times}  
\usepackage{helvet} 
\usepackage{courier}  
\usepackage[hyphens]{url}  
\usepackage{graphicx} 
\usepackage{caption} 
\usepackage{hyperref} 

\usepackage[switch]{lineno}  

\usepackage{algorithm}
\usepackage[noend]{algpseudocode}
\usepackage{amsmath}
\usepackage{amssymb}
\usepackage{amsthm}
\newtheorem{theorem}{Theorem}
\newtheorem{definition}{Definition}
\newtheorem{lemma}{Lemma}
\usepackage{tikz}
\usepackage{subcaption}

\def \R {\mathbb R}
\def \Z {\mathbb Z}  

\def \x {\mathbf{x}} 
\def \xc {\x_c} 
\def \xd {\x_d} 
\def \X {X} 
\def \Xc {\X_c} 
\def \Xd {\X_d} 
\def \dc {d_c} 
\def \dd {d_d} 
\def \wc {{{\mathbf v}}} 
\def \wdd {{\mathbf w}} 
\def \obj {f} 
\def \surr {g} 
\def \y {y} 
\def \noise {\epsilon} 
\def \budget {N} 
\def \nb {D} 
\def \it {n} 
\def \k {k} 
\algrenewcommand\algorithmicrequire{\textbf{Input}}
\algrenewcommand\algorithmicensure{\textbf{Output}}

\title{Black-box Mixed-Variable Optimisation Using a Surrogate Model that Satisfies Integer Constraints}
\author{
    Laurens Bliek,
  Arthur Guijt,
Sicco Verwer,
Mathijs de Weerdt\\
}
\date{
Delft University of Technology,\\
              Faculty of Electrical Engineering, Mathematics and Computer Science, \\
              Van Mourik Broekmanweg 6, 
              2628 XE Delft,
              The Netherlands\\
              \url{l.bliek@tudelft.nl}
              }
\begin{document}


\maketitle

\begin{abstract}
A challenging problem in both engineering and computer science is that of minimising a function for which we have no mathematical formulation available, that is expensive to evaluate, and that contains continuous and integer variables, for example in automatic algorithm configuration. 
Surrogate-based algorithms are very suitable for this type of problem, but most existing techniques are designed with only continuous or only discrete variables in mind. 
Mixed-Variable ReLU-based Surrogate Modelling (MVRSM) is a surrogate-based algorithm 
that uses a linear combination of rectified linear units,
defined in such a way that (local) optima satisfy the integer constraints. 
This method
outperforms
the state of the art on several synthetic benchmarks with up to $238$ continuous and integer variables, and achieves competitive performance on two real-life benchmarks: XGBoost hyperparameter tuning and Electrostatic Precipitator optimisation.
\end{abstract}

\section{Introduction}

Surrogate modelling techniques such as Bayesian optimisation have a long history of success in optimising expensive black-box objective functions~\cite{movckus1975bayesian,jones1998efficient,mockus2012bayesian}.
These are functions that have no mathematical formulation available and take some time or other resource to evaluate, which occurs for example when they are the result of some simulation, algorithm or scientific experiment.
Often there is also randomness or noise involved in these evaluations.
By approximating the objective with a cheaper surrogate model, the optimisation problem can be solved more efficiently.

While most attention in the literature has gone to problems in continuous domains, recently solutions for combinatorial optimisation problems have started to arise~\cite{garrido2020dealing,baptista2018bayesian,bartz2017model,ueno2016combo,bliek2019black}.
Yet many problems contain a mix of continuous and discrete variables, for example material design~\cite{iyer2019data}, optical filter optimisation~\cite{yang2019towards},
and automated machine learning~\cite{hutter2019automated}.
The literature on surrogate modelling techniques for these types of problems is even more sparse than for purely discrete problems.
Discretising the continuous variables to make use of a purely discrete surrogate model, or applying rounding techniques to make use of a purely continuous surrogate model are both seen as common but inadequate ways to solve the problem~\cite{garrido2020dealing,ru2019bayesian}.
The few existing techniques that can deal with a mixed variable setting still have considerable room for improvement in accuracy or efficiency.
When the surrogate model is not expressive enough and does not model any interaction between the different variables, it will perform poorly, especially when many variables are involved.
On the other hand, most Bayesian optimisation techniques do model the interaction between all variables, but use a surrogate model that grows in size every iteration.
This causes those algorithms to become slower over time, potentially even becoming more expensive than the expensive objective itself.

Our main contribution is a surrogate modelling algorithm called Mixed-Variable ReLU-based Surrogate Modelling (MVRSM) that can deal with problems with continuous and integer variables efficiently and accurately.
This is realised by using a continuous surrogate model that:
\begin{itemize}
    \item models interactions between all variables,
    \item does not grow in size over time and can be updated efficiently, and
    \item has local optima that are located exactly in points of the search space where the integer constraints are satisfied.
\end{itemize}
The first point ensures that the model remains accurate, even for large-scale problems.
The second point ensures that the algorithm does not slow down over time.
Finally, the last point eliminates the need for rounding techniques, and also eliminates the need for repeatedly using integer programming as is done in~\cite{daxberger2019mixed}.

Besides the proposed algorithm, the contributions include a proof 
that the local optima of the proposed surrogate model are integer-valued in the intended variables.
We also include an experimental proof of the effectiveness of this method 
on several 
large-scale
synthetic benchmarks from related work
and on two real-life benchmarks: XGBoost hyperparameter tuning and Electrostatic Precipitator optimisation.

\section{Preliminaries}

This work considers the problem of finding the minimum of a mixed-variable black-box objective function $\obj: \R^{\dc} \times \Z^{\dd} \rightarrow \R$ that can only be accessed via expensive and noisy measurements 
$\y = \obj(\xc,\xd)+\noise$.
That is, we want to solve

\begin{align}
    \min_{\xc\in\Xc,\xd\in\Xd} &  f(\xc,\xd),
    \label{eq:mainproblem}
\end{align}

where $\dc$ is the number of continuous variables, $\dd$ the number of integer variables, $\noise \in \R$ is a zero-mean random variable with finite variance, and $\Xc \subseteq \R^{\dc}$ and $\Xd \subseteq \Z^{\dd}$ are the bounded domains of the continuous and integer variables respectively.
In this work, the lower and upper bounds of either $\Xc$ or $\Xd$ for the $i$-th variable are denoted $l_i$ and $u_i$ respectively.
Since $\Xd\subseteq \Z^{\dd}$, we call $\xd \in \Z^{\dd}$ the integer constraints.
Expensive in this context means that it takes some time or other resource to evaluate $\y$, as is the case in for example hyperparameter tuning problems~\cite{bergstra2013making} and many engineering problems~\cite{DONEpaper,ueno2016combo}.
Therefore, we wish to solve~\eqref{eq:mainproblem} using as few samples as possible.

The problem 
is 
usually solved with a surrogate modelling technique such as Bayesian optimisation~\cite{mockus2012bayesian}.
In this approach, the data samples 
$(\xc, \xd, \y)$
are used to approximate the objective $\obj$ with a surrogate model $\surr$.
Usually, $\surr$ is a machine learning model such as a Gaussian process, random forest or a weighted sum of nonlinear basis functions.
In any case, it has an exact mathematical formulation, which means that $\surr$
can be optimised with standard techniques as it is not expensive to evaluate and it is not black-box.
If $\surr$ is indeed a good approximation of the original objective $\obj$, it can be used to suggest new candidate points of the search space $\Xc\times\Xd$ where $\obj$ should be evaluated.
This happens iteratively, where in every iteration $\obj$ is evaluated, the approximation $\surr$ of $\obj$ is improved, and optimisation on $\surr$ is used to suggest a next point to evaluate $\obj$.

\section{Related work}

In Bayesian optimisation, Gaussian processes are the most popular surrogate model~\cite{mockus2012bayesian}.
On the one hand, these surrogate models lend themselves well to problems with only continuous variables, but not so much when they include integer variables as well.
On the other hand, there have been several recent approaches to develop surrogate models for problems with only discrete variables~\cite{garrido2020dealing,baptista2018bayesian,ueno2016combo,bliek2019black}.

The mixed-variable setting is not as well-developed, although
there are some surrogate modelling methods
that can deal with this.
We start by mentioning two well-known methods, namely SMAC~\cite{hutter2011sequential} and HyperOpt~\cite{bergstra2013making}, followed by more recent work, along with their strengths and shortcomings.
We end this section with recent work on discrete surrogate models that we make use of throughout this paper.

SMAC~\cite{hutter2011sequential} uses random forests as the surrogate model.
This captures interactions between the variables nicely, but the main disadvantage is that the random forests are less accurate in unseen parts of the search space, at least compared to other surrogate models.
HyperOpt~\cite{bergstra2013making} uses a Tree-structured Parzen Estimator as the surrogate model.
This algorithm is known to be fast in practice, has been shown to work in settings with over $200$ variables, and also has the ability to deal with conditional variables, where certain variables only exist if other variables take on certain values.
Its main disadvantage is that complex interactions between variables are not modelled.
Most other existing Bayesian optimisation algorithms have to resort 
to rounding or discretisation
in order to deal with the mixed variable setting, which both have their disadvantages~\cite{garrido2020dealing,ru2019bayesian}.

More recently, the CoCaBO algorithm was proposed~\cite{ru2019bayesian}, which is developed for problems with a mix of continuous and categorical variables.
It makes use of a mix of multi-armed bandits and Gaussian processes.
Another interesting new research direction is to combine the advantages of Gaussian processes and artificial neural networks~\cite{kim2020surrogate}, although more research is required to make this computationally feasible for larger problems.
Other research groups have focused their attention to multi-objective mixed-variable problems 
~\cite{yang2019towards,iyer2019data}.

Most of the methods mentioned here suffer from the drawback that the surrogate model grows while the algorithm is running, causing the algorithms to slow down over time.
This problem has been addressed and solved for the continuous setting in the DONE algorithm~\cite{DONEpaper} and for the discrete setting in the COMBO~\cite{ueno2016combo} and IDONE algorithms~\cite{bliek2019black} by making use of parametric surrogate models that are linear in the parameters.
The 
MiVaBO algorithm~\cite{daxberger2019mixed} is, to the best of our knowledge, the first algorithm that applies this solution to the mixed variable setting.
It relies on an alternation between continuous and discrete optimisation to find the optimum of the surrogate model.

In contrast with MiVaBO, 
the IDONE algorithm has
the theoretical guarantee that any local minimum of the surrogate model satisfies the integer constraints, so only continuous optimisation needs to be used.
This is achieved by using a surrogate model consisting of a linear combination of rectified linear units (ReLUs), a popular basis function in the machine learning community.
Using only continuous optimisation is much more efficient than the approach used in MiVaBO.
However, this theory only applies to problems without continuous variables.

\section{Mixed-Variable ReLU-based Surrogate Modelling}\label{sec:alg}


In this section, we use the theory from
the IDONE algorithm
to develop a ReLU-based surrogate model for the mixed-variable setting.
This is far from trivial, as a wrong choice of surrogate model might result in limited interaction between all variables, in not being able to optimise the surrogate model efficiently, or in not being able to satisfy the integer constraints.

Below we present
the Mixed-Variable ReLU-based Surrogate Modelling (MVRSM) algorithm.
This algorithm makes use of
a surrogate model based on rectified linear units
and includes
interactions between all variables, is easy to update and to optimise, and has its local optima situated in points that satisfy the integer constraints.

\subsection{Proposed surrogate model}

As in related work~\cite{CDONEpaper,bliek2019black,daxberger2019mixed}, we use a continuous surrogate model $\surr:\R^{\dc+\dd}\rightarrow \R$:

\begin{align}
    \surr(\xc,\xd) & = \sum_{\k=1}^{\nb} c_\k \phi_\k(\xc,\xd), \label{eq:surrdef}
\end{align}
with $\nb$ being the number of basis functions.
The model is linear in its own parameters $c$, which  allows it to be trained with linear regression.
We choose the basis functions $\phi$ in such a way that all local optima $(\xc^*,\xd^*)$ of the model satisfy $\xd \in \Z^{\dd}$, as explained later in this section.
This leads to an efficient way of finding the minimum of the surrogate model for mixed variables.
We choose rectified linear units as the basis functions:

\begin{align}
    \phi_\k(\xc,\xd) & =\max\{0,z_\k(\xc,\xd)\},\label{eq:basisf}\\
    z_\k(\xc,\xd) & = [\wc_{\k}^T \wdd_{\k}^T]\left[\begin{array}{l}\xc\\ \xd\end{array}\right] + b_\k, \label{eq:zf} 
\end{align}
with $\wc_{\k}\in \R^{\dc}$, $\wdd_{\k} \in \R^{\dd}$, and $b_k\in \R$. 
This causes the surrogate model $\surr$ to be piece-wise linear.
There are four strategies for choosing the model parameters $\wc_\k, \wdd_\k, b_\k$: 
\begin{itemize}
    \item optimise them together with the weights $c_\k$, 
    \item choose them directly according to the data samples in a non-parametric way using kernel basis functions~\cite{mockus2012bayesian,ru2019bayesian},
    \item choose them randomly once and then fix them~\cite{DONEpaper,CDONEpaper,ueno2016combo,daxberger2019mixed}, or
    \item choose them according to the variable domains $\Xc, \Xd$ and then fix them~\cite{bliek2019black}.
\end{itemize}

The first option is not recommended as nonlinear optimisation would have to be used, while linear regression techniques can be used for the parameters $c_\k$.
The second option has the downside that more and more basis functions need to be added as data samples are gathered, making the surrogate model grow in size while the algorithm is running.
This is what happens in most Bayesian optimisation algorithms, which causes them to slow down over time.
The third option fixes this problem, but even though there are good approximation theorems available for a random choice of the parameters~\cite{rahimi2008uniform,DONEpaper}, it does not give any guarantees on satisfying the integer constraints.
The fourth option does, but only for problems that have no continuous variables.
Therefore, we propose to use a mix of the third and fourth option, getting the best of both options, as explained below.


We first state the required definitions, followed by our main theoretical contribution.

\begin{definition}[Integer $z$-function]
An \emph{integer $z$-function} $z_\k$ is chosen according to~\eqref{eq:zf} with $\wc = \mathbf 0$ and with $\wdd$ and $b$ having integer values chosen according to Algorithm $2$ from~\cite{bliek2019black}.
That means it has one of the following forms: $z_\k(\xc,\xd) = z_\k(\xd)=\pm (x_i - \alpha)$, with $x_i$ an element from $\xd$ and $\alpha\in \Z$ chosen between $l_i$ and $u_i$ (the lower and upper bounds of $x_i$), or $z_\k(\xc,\xd) =z_\k(\xd) = \pm (x_i - x_{i-1} - \alpha )$,
for $i>1$
and $\alpha \in \Z$ chosen between $l_i - u_{i-1}$ and $u_i - l_{i-1}$.
This results in a basis function that depends only on one or two subsequent integer variables and does not depend on any continuous variables.
\end{definition}

By making use of the integer $z$-functions, we have a surrogate model with basis functions that depend on the integer variables.
If we would add basis functions that depend only on the continuous variables, the possible interaction between continuous and integer variables would not be modelled.
But if we add randomly chosen mixed basis functions as in~\cite{daxberger2019mixed}, then we might lose the guarantee that integer constraints are satisfied in local minima.
See Figure~\ref{fig:exz} (left).
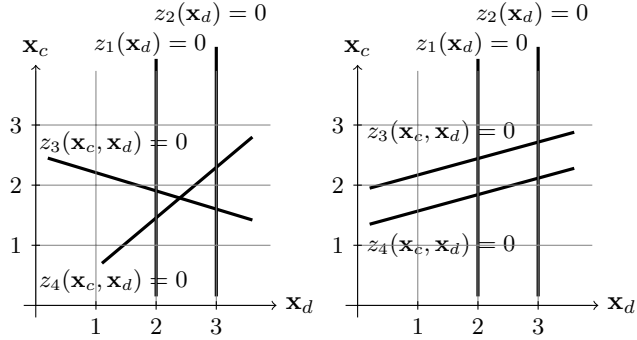
\begin{figure}[tb]
\centering
\begin{tikzpicture}[scale=0.8]
  \draw[very thick] (2.0,0.15) -- (2.0,4.1);
  \draw (1.9,4) node[above] {\small $z_1(\xd)=0$};
  \draw[very thick] (3.0,0.15) -- (3.0,4.3);
  \draw (2.9,4.5) node[above] {\small $z_2(\xd)=0$};
  \draw[very thick] (0.2,2.45) -- (3.6,1.42);
  \draw (-0.1,2.7) node[right] {\small $z_3(\xc,\xd)=0$};
  \draw[very thick] (1.1,0.7) -- (3.6,2.8);
  \draw (-0.1,0.4) node[right] {\small $z_4(\xc,\xd)=0$};

  \draw[style=help lines] (0,0) grid (3.9,3.9);

  \draw[->] (-0.2,0) -- (4,0) node[right] {$\xd$};
  \draw[->] (0,-0.2) -- (0,4) node[above] {$\xc$};

  \foreach \x/\xtext in {1/1, 2/2, 3/3}
    \draw[shift={(\x,0)}] (0pt,2pt) -- (0pt,-2pt) node[below] {\small $\xtext$};

  \foreach \y/\ytext in {1/1, 2/2, 3/3}
    \draw[shift={(0,\y)}] (2pt,0pt) -- (-2pt,0pt) node[left] {\small  $\ytext$};

\end{tikzpicture}%
\begin{tikzpicture}[scale=0.8]
  \draw[very thick] (2.0,0.15) -- (2.0,4.1);
  \draw (1.9,4) node[above] {\small $z_1(\xd)=0$};
  \draw[very thick] (3.0,0.15) -- (3.0,4.3);
  \draw (2.9,4.5) node[above] {\small $z_2(\xd)=0$};
  \draw[very thick] (0.2,1.95) -- (3.6,2.88);
  \draw (0.0,2.88) node[right] {\small $z_3(\xc,\xd)=0$};
  \draw[very thick] (0.2,1.35) -- (3.6,2.28);
  \draw (0.0,1.0) node[right] {\small $z_4(\xc,\xd)=0$};

  \draw[style=help lines] (0,0) grid (3.9,3.9);

  \draw[->] (-0.2,0) -- (4,0) node[right] {$\xd$};
  \draw[->] (0,-0.2) -- (0,4) node[above] {$\xc$};

  \foreach \x/\xtext in {1/1, 2/2, 3/3}
    \draw[shift={(\x,0)}] (0pt,2pt) -- (0pt,-2pt) node[below] {\small $\xtext$};

  \foreach \y/\ytext in {1/1, 2/2, 3/3}
    \draw[shift={(0,\y)}] (2pt,0pt) -- (-2pt,0pt) node[left] {\small $\ytext$};

\end{tikzpicture}
\caption{\textbf{(left)} Example of the problem with mixed basis functions for $1$ integer ($\xd$) and $1$ continuous variable ($\xc$).
All local minima are located in points where two 
lines intersect.
This works fine for the intersections with the integer $z$-functions $z_1$, $z_2$, but not for the two randomly chosen $z$-functions $z_3$, $z_4$, as in that point $\xd$ takes on a non-integer value.
\textbf{(right)} A solution to the problem is to use mixed $z$-functions that are parallel to a number of linearly independent vectors equal to $\dc$.
This ensures that all intersections are located in points where $\xd$ is integer.}
\label{fig:exz}
\end{figure}

To avoid both 
problems, we propose to add mixed basis functions as in~\cite{daxberger2019mixed}, but we choose them pseudo-randomly rather than randomly.
This benefits from the success that randomly chosen weights have had in the past~\cite{DONEpaper,CDONEpaper,ueno2016combo,daxberger2019mixed},
while avoiding the problem from Figure~\ref{fig:exz} (left).

\begin{definition}[Mixed $z$-function]\label{def:mixed}
A \emph{mixed $z$-function} $z_\k$ is chosen according to~\eqref{eq:zf} with $\omega_\k = \left[\begin{array}{l}\wc_\k \\ \wdd_\k \end{array}\right]$ sampled from a set $\Omega$ that contains $\dc$ random vectors in $\R^{\dc+\dd}$ with a continuous probability distribution $p_{\omega}$, and $b_\k$ is then chosen from a random continuous probability distribution $p_b$ which depends on $\omega_k$. This results in a basis function that depends on all continuous and on all integer variables.
\end{definition}

The probability distributions $p_\omega$ and $p_b$  are chosen in such a way that the mixed $z$-functions are never completely outside the domain $\Xc\times\Xd$.
(The exact procedure for choosing them can be found in the appendix.)
As a result of the definition, all mixed $z$-functions will be parallel to one of the $\dc$ random vectors. 
See Figure~\ref{fig:exz} (right).
This gives the following result, which guarantees the unique property of this continuous surrogate model, i.e. that all local minima are integer-valued in the intended variables:

\begin{theorem}\label{thm:main}
If the surrogate model $\surr$ consists entirely of integer and mixed $z$-functions, then any strict local minimum $(\xc^*,\xd^*)$ of $\surr$ satisfies $\xd \in \Z^{\dd}$.

\end{theorem}

This result
makes it possible to apply
continuous optimisation
to find a minimum of our surrogate model, instead of having to solve a mixed-integer program which is more expensive, or having to resort to rounding which is sub-optimal.
As the rectified linear units are linear almost everywhere, the surrogate model can be optimised relatively easily with a gradient-based technique such as L-BFGS~\cite{nocedal} or other standard methods.

Before presenting the proof, we 
state two results 
that are relevant to our approach:

\begin{lemma}\label{thm:localminima}
Any strict local minimum of $\surr$ is located in a point $(\xc^*,\xd^*)$ with\linebreak $z_\k(\xc^*,\xd^*) = 0$ for  $(\dc+\dd)$ linearly independent functions $z_\k$~\cite{bliek2019black}.
\end{lemma}

This follows from the fact that $\surr$ is piece-wise linear, so any strict local minimum must be located in a point where the model is nonlinear in all directions.

\begin{lemma}\label{lem:integers}
If $z_\k(\xd) = 0$ for $\dd$ different linearly independent integer $z$-functions $z_\k$,
then $\xd \in \Z^{\dd}$.
\end{lemma}

\begin{proof}
The proof follows exactly the same reasoning as the proof of~\cite[Thm. 2 (II)]{bliek2019black}. 
\end{proof}


We now show the proof of Theorem~\ref{thm:main} below.
\begin{proof}[Proof of Theorem~\ref{thm:main}]
From Lemma~\ref{thm:localminima}
it follows that $z_\k(\xc^*,\xd^*) = 0$ for $\dc+\dd$ linearly independent $z_\k$. Since all mixed $z$-functions are parallel to one of the $\dc$ randomly chosen vectors, there can only be $\dc$ linearly independent mixed $z$-functions. As all other $z$-functions are integer $z$-functions, this means that there are $\dd$ linearly independent integer $z$-functions. The result now follows from Lemma
~\ref{lem:integers}%
. 
\end{proof}

\subsection{MVRSM details}

In the proposed algorithm, we first initialise the model by adding basis functions consisting of integer and mixed $z$-functions.
The procedure of generating integer $z$-functions is the same as in the advanced model of~\cite{bliek2019black}, which gives $D_d = 1+4|\Xd|-|{\Xd}[1]|-|{\Xd}[\dd]|$ basis functions in total, with $\Xd[i]$ the domain of the $i$-th integer variable.
We then generate $D_c$ mixed $z$-functions.
Since our approach allows us to choose any number of mixed $z$-functions without losing the guarantee of satisfying the integer constraints, computational resources are the only limiting factor here.
We choose $D_c=\lceil\dc\cdot D_d/\dd\rceil$ to have the same number of mixed $z$-functions per continuous variable as the number of integer $z$-functions per integer variable. 

The algorithm proceeds with an iterative procedure consisting of four steps%
: \textbf{1)} evaluating the objective, \textbf{2)} updating the model, \textbf{3)} finding the minimum of the model, and \textbf{4)} performing an exploration step.
Evaluating the objective $\obj$
gives a data sample
$(\xc, \xd, \y)$.
The update procedure of the surrogate model is performed with the recursive least squares algorithm~\cite{sayed1998recursive}, which can be done since the model is linear in its parameters $c_\k$.
We also add a regularisation factor of $10^{-8}$ here for numerical stability. Furthermore, the weights $c_\k$ from~\eqref{eq:surrdef} are initialised as $c_\k=1$ for the basis functions corresponding to integer $z$-functions, and as $c_\k=0$ for the basis functions corresponding to the mixed $z$-functions.
The minimum of the model is found with the L-BFGS method~\cite{nocedal}, which is improved by giving an analytical representation of the Jacobian.
For this purpose, we define $[\frac{d}{dx} \max\{0,x\}](0) = 0.5$, as the rectified linear units are non-differentiable in $0$.
We run the L-BFGS method for $20$ sub-iterations only, as the goal is not to find the exact minimum of the surrogate model, but rather to find a promising area of the search space.
Lastly, we perform an exploration step on the point $(\xc^*, \xd^*)$ found by the L-BFGS algorithm, where the point is given a small perturbation so that local optima can be avoided.
The whole algorithm is shown in Algorithm~\ref{alg:MVRSM}.

\begin{algorithm}[htbp]
\caption{MVRSM algorithm}\label{alg:MVRSM}
 \begin{algorithmic}
 \Require Objective $\obj$, domains $\Xc$, $\Xd$, budget $\budget$
 \Ensure $\xc^{(\budget)}, \xd^{(\budget)}$, $\y^{(\budget)}$
 \State Initialise surrogate $\surr$ with integer and mixed $z$-functions
 \State Initialise $c_\k=1$ for integer $z$-functions and $c_\k=0$ for mixed $z$-functions, initialise other recursive least squares parameters
 \For{$\it=1, \ldots, \budget$}
    \State Evaluate $\y^{(\it)} = \obj\left(\xc^{(\it)},\xd^{(\it)}\right)+\noise$
    \State Update the parameters of $\surr$ with data point $\left(\xc^{(\it)},\xd^{(\it)},  \y^{(\it)}\right)$ using recursive least squares
    \State Solve $\min \surr(\xc,\xd)$ over domains $\Xc$, $\Xd$ with relaxed integer constraints using \mbox{L-BFGS}
    \State Explore around the found solution $(\xc^*,\xd^*)$ by adding random perturbation \mbox{$\left(\delta_c,\delta_d\right)\in\R^{\dc}\times\Z^{\dd}$}: $\left(\xc^{(\it+1)},\xd^{(\it+1)}\right) = (\xc^*,\xd^*)+(\delta_c,\delta_d)$
 \EndFor
\end{algorithmic}
\end{algorithm}

\section{Experiments}\label{sec:exp}

To see if the proposed algorithm overcomes the drawbacks of existing surrogate modelling algorithms for problems with mixed variables in practice, 
we compare MVRSM with different state-of-the-art methods and random search on two real-life benchmarks and on several synthetic benchmark functions used in related work.
For the real-life benchmarks we consider one from machine learning and one from engineering, namely XGBoost hyperparameter tuning and Electrostatic Precipitator (ESP) optimisation.
For the synthetic benchmarks we consider mixed-variable problems of up to $238$ variables
from related literature.

For comparison with other methods, we consider state-of-the-art surrogate modelling algorithms that are able to deal with a mixed-variable setting, have code available, and are concerned with single-objective problems.
We compare our method with HyperOpt~\cite{bergstra2013making} (HO) and SMAC~\cite{hutter2011sequential} as two popular and established surrogate modelling algorithms that can deal with mixed variables, and we compare with CoCaBO~\cite{ru2019bayesian} as a more recent method that can deal with a mix of continuous and categorical variables.
As is good practice in surrogate modelling, we include random search (RS) in the comparisons to confirm whether more sophisticated methods are even necessary.
For the same reason, we include a standard Bayesian optimisation (BO) algorithm,
where we use rounding on the integer variables when calling the objective function.

Though we consider MiVaBO~\cite{daxberger2019mixed} also to be part of the state of the art, at the time of writing the authors have not made their code available yet.
We still include their benchmarks in the comparison. 
We make no comparison with multi-fidelity methods such as Hyperband~\cite{li2017hyperband} or BOHB~\cite{falkner2018bohb}, as these methods can only be applied to our hyperparameter tuning benchmark and not to the other benchmarks.
We also did not compare with the multi-objective methods from the related work section, as we did not find a way to make a fair comparison for single-objective problems, even though they were specifically developed for the mixed-variable setting.
Because MiVaBO uses a more expensive optimisation method, we expect MVRSM to outperform not only multi-objective methods but also MiVaBO on single-objective domains in terms of efficiency, but further research is required to confirm this.

\subsection{Implementation details}


To enable the use of categorical variables in MVRSM, we convert those variables to integers.
To enable the use of integer or binary variables for CoCaBO, we convert those variables to categorical variables.
For CoCaBO, we chose a mixture weight~\cite[Eq. (2)]{ru2019bayesian} of $0.5$ as this seemed to give the best results on synthetic benchmarks. 
SMAC is put in deterministic mode instead of the default, as this improved the results in all of our experiments: the default often repeats function evaluations at the same location, leading to an inefficient method.
The random search uses HyperOpt's implementation.
The code for HyperOpt%
\footnote{
\url{https://github.com/hyperopt/hyperopt}
}, %
SMAC%
\footnote{\url{https://github.com/automl/SMAC3}}, %
CoCaBO%
\footnote{
\url{https://github.com/rubinxin/CoCaBO_code}
}, and 
MVRSM%
\footnote{
\url{https://github.com/lbliek/MVRSM}
}
is availabe online.
For Bayesian Optimisation we use an existing  implementation\footnote{\url{https://github.com/fmfn/BayesianOptimization}} which uses Gaussian processes with the Upper Confidence Bound acquisition function.
Experiments were done in Python on 
an {Intel(R) Xeon(R) Gold 6148 CPU @ 2.40GHz} with 32 GB of RAM, and each experiment was performed using only a single CPU core.
In line with~\cite{ru2019bayesian}, all methods start with $24$ initial random guesses, which are not shown in the figures.
We used each algorithm's own implementation for this, but made sure to set it to the same uniform probability distribution over the whole search space.

All methods are compared using the same number of iterations, and the best function value found at each iteration is reported, averaged over multiple runs.
The standard deviations are indicated with shaded areas in the relevant figures.
The computation time of the methods is also reported for every iteration.



\subsection{Results on XGBoost hyperparameter tuning}

First, we consider a problem similar to that of hyperopt-sklearn~\cite{komer2014hyperopt}, where hyperparameters for a preprocessing method as well as for a classifier need to be selected and tuned simultaneously.
The choice of classifier is limited to the XGBoost method only~\cite{chen2016xgboost}, which has several hyperparameters of different shapes (continuous, integer, binary, categorical, and conditional).\footnote{The hyperparameters for XGBoost can be found at \url{https://xgboost.readthedocs.io/en/latest/parameter.html\#learning-task-parameters}}%

Conditional variables only exist when other variables take on certain values.
SMAC and HO can both deal with these efficiently, but for the other methods we use a na\"ive encoding where these variables still exist but do not influence the objective function if other choices are made.
Together with the hyperparameters for preprocessing, there are $7$ integer, $11$ continuous, and over $116$ categorical/binary/conditional variables.
The preprocessing method and XGBoost are applied to the steel-plates-faults dataset\footnote{\url{https://archive.ics.uci.edu/ml/datasets/Steel+Plates+Faults}}, and the objective is the result of a $5$-fold cross-validation, multiplied by $-1$ to make it a minimisation problem.
To find not just accurate but also efficient hyperparameters, we set a time limit of $8$ seconds, chosen roughly equal to twice the time it takes when using default hyperparameters.
If the objective took longer than that to evaluate, an objective value of $0$ was returned.
On average, the evaluation of the objective took just over $3$ seconds on our hardware.

\begin{figure}[tb]
\centering
\centering
\includegraphics[width=0.99\columnwidth]{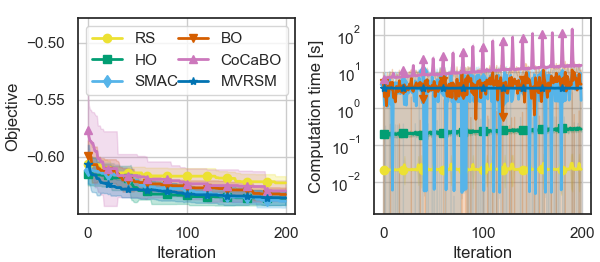}
\caption{Results on the XGBoost hyperparameter tuning benchmark ($7$ integer, $11$ continuous, ${>}116$ categorical/binary/conditional),
averaged over $7$ runs.} \label{fig:HPO}
\end{figure}

Figure~\ref{fig:HPO} shows the results on this benchmark for $200$ iterations, averaged over $10$ runs.
MVRSM gets a similar performance as its competitors on this problem, ending up with an average objective of $-0.637$.
A pair-wise Student's T-test on the final iteration shows no significant difference between MVRSM and the other surrogate-based methods ($p>0.05$), though it outperforms random search ($p\approx 0.003$).

It is important to note that besides
random search, MVRSM is the only method that has a fixed computation time per iteration.
All other methods (except SMAC, as shown later in this paper) become slower over time. 
This is especially important for problems where the evaluation time of the objective takes a similar time as the surrogate-based algorithm, e.g. $10$ seconds or less for CoCaBO, which is the case for this hyperparameter tuning problem.
In this case it is not possible anymore to disregard the computation time of the algorithm, even though this is often done in literature.
Furthermore, CoCaBO tunes its own hyperparameters every $10$ iterations, which costs even more computational resources.
In contrast, MVRSM has quite a low number of hyperparameters, and we choose them the same way in all reported experiments.
This makes it much easier to apply than other methods, or in the case of CoCaBO, much more efficient.
The practical use of this fact should not be underestimated, as especially on hyperparameter tuning problems one wants to avoid having to tune the hyperparameters of the surrogate-based algorithm.

\subsection{Results on Electrostatic Precipitator optimisation}

The ESP problem~\cite{rehbach2018ESPBenchmark} is a real-life industrial problem where components of a gas cleaning system need to be designed.
The goal is to reduce environmental pollution.
The system contains $49$ different slots that can each hold one of $8$ different types of metal plates that each influence the gas flow in a different way.
After choosing the configuration of the plates, an expensive computational fluid dynamics simulator calculates the corresponding objective, taking around $27$ seconds on average on our hardware.
This problem has $8$ categories for each variable, though $5$ of the categories correspond to ordinal variables, namely the size of holes in the metal plates.

We have adapted the ESP problem such that the $5$ hole sizes are not restricted to fixed values, but are free to take on different continuous values.
This adds $5$ continuous variables to the problem with otherwise only categorical variables, using the same five options for each slot, as having each slot take on a different value would substantially increase the manufacturing costs.

\begin{figure}[tb]
\centering
\includegraphics[width=0.99\columnwidth]{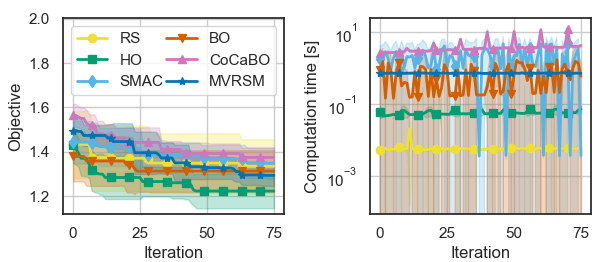}
\caption{Results on the ESP benchmark ($49$ categorical, $5$ continuous), averaged over $5$ runs.} \label{fig:ESP3}
\end{figure}

Figure~\ref{fig:ESP3} shows the results on this benchmark for $76$ iterations, as the problem typically has a budget of $100$ function evaluations~\cite{GECCO2020industrial} and we used $24$ of them for random initial guesses.
MVRSM ends up with an average objective of $1.29$.
A pair-wise Student's T-test on the final iteration shows no significant difference between MVRSM and the other methods ($p\approx 0.13$ when compared with HO), except when comparing with CoCaBO ($p\approx 0.024$) which performs more poorly on this problem.
This indicates that MVRSM is a competitive method in realistic expensive optimisation problems.
However, the effect of slowdown in the other algorithms is not clearly visible due to the low number of iterations used.
The real life benchmarks are too expensive to evaluate for a large number of iterations, which is why we now turn to investigate synthetic benchmarks.
Besides a larger number of function evaluations, the use of synthetic benchmarks also allow us to investigate the performance of MVRSM on large-scale problems.

\subsection{Results on relevant synthetic benchmarks}

To investigate the effect of algorithms slowing down, as well as the scalability of MVRSM and how it compares to other algorithms on their own benchmarks,
we make a comparison on several large-scale synthetic functions from related literature.
The Ackley and Rosenbrock functions are two well-known benchmarks in the black-box optimisation community\footnote{Details available at \url{https://www.sfu.ca/~ssurjano/optimization.html}
}.
Both can be scaled to any dimension.
For the Ackley function we choose a dimension of $53$, but $50$ of the variables were adapted to binary variables in $\Xd=\{0,1\}^{50}$.
The $3$ continuous variables were limited to $\Xc=[-1,1]^3$.
This causes the problem to be of a similar scale as the problem of variational auto-encoder hyperparameter tuning after binarising the discrete hyperparameters~\cite[App. E.1]{daxberger2019mixed}.
For the Rosenbrock function we choose a dimension of $239$, 
with the first $119$ variables adapted to integers in $\Xd = \{-2, -1, 0, 1, 2\}^{119}$, and $119$ continuous variables limited to $\Xc = [-2,2]^{119}$. 
The function was scaled with a factor $1/50000$.
This problem is of the same scale as the problem of feed-forward classification model hyperparameter tuning~\cite{bergstra2013making}, except that the ratio between continuous and integer variables is chosen to be $1:1$.
Uniform noise in~$[0,10^{-6}]$ was added to each function evaluation in both functions.
Finally, we investigated a randomly generated synthetic test function from~\cite[Appendix C.1, Gaussian weights variant]{daxberger2019mixed}.
We scaled this problem up to have $119$ integer and $119$ continuous variables.
No bounds were reported for this problem so we set them to $\Xd=\{0,1,2,3\}^{119}$ for the integer variables and $\Xc=[0,3]^{119}$ for the continuous variables.


\begin{figure}[tb]
\centering
\includegraphics[width=0.99\columnwidth]{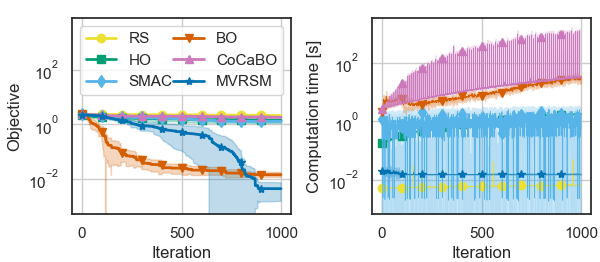}
\caption{Results on the Ackley53 benchmark
($50$ binary, $3$ continuous), 
averaged over $7$ runs. Note that the left figure has a logarithmic scale. This problem is of a similar scale as variational auto-encoder hyperparameter tuning~\cite[Sec. 4.2]{daxberger2019mixed}.} \label{fig:A53}
\end{figure}

\begin{figure}[tb]
\centering
\includegraphics[width=0.99\columnwidth]{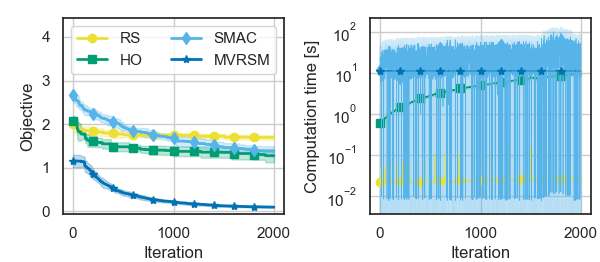}
\caption{Results on the Rosenbrock238 benchmark
($119$ integer, $119$ continuous),
averaged over $7$ runs. 
BO and CoCaBO were not evaluated for this benchmark due to the large computation time. 
This problem is of a similar scale as feed-forward classification model hyperparameter tuning~\cite{bergstra2013making}.
} \label{fig:R238}
\end{figure}

\begin{figure}[tb]
\centering
\includegraphics[width=0.99\columnwidth]{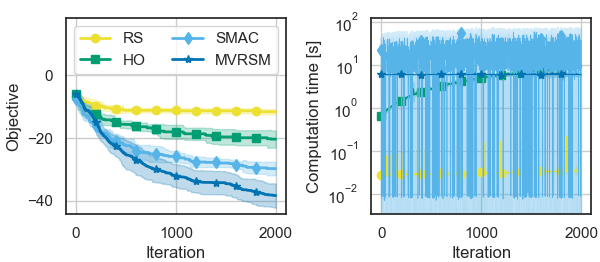}
\caption{Results on one randomly generated MiVaBO synthetic benchmark~\cite[Appendix C.1, Gaussian weights variant]{daxberger2019mixed} with a larger scale
($119$ integer, $119$ continuous), 
averaged over $7$ runs.
BO and CoCaBO were not evaluated for this benchmark due to the large computation time.
This problem is of a similar scale as feed-forward classification model hyperparameter tuning~\cite{bergstra2013making}.
} \label{fig:LM238}
\end{figure}

Figures~\ref{fig:A53}-\ref{fig:LM238} show the performance of the different algorithms on these three benchmarks.
MVRSM clearly outperforms the other methods in terms of accuracy,
and the computation times of BO and CoCaBO become prohibitively large.
The slowdown of the other surrogate-based algorithms is now clearly visible, with their computation time increasing every iteration, although SMAC does not suffer from this.

The fact that MVRSM outperforms both HO and SMAC is surprising, considering that the scale of the larger problems is similar to that of one of HO's own benchmarks, while the authors of HO consider SMAC a potentially superior optimiser~\cite[p. 8]{bergstra2013making}.

\section{Conclusion and Future Work}

We showed how Mixed-Variable ReLU-based Surrogate Modelling (MVRSM) solves three problems present in methods that can deal with mixed variables in expensive black-box optimisation.
First, it solves the problem of slowing down over time due to a growing surrogate model.
Second, it solves the problem of sub-optimality and inefficiency that may arise due to the need to satisfy integer constraints.
Third, it solves the problem of model inaccuracies due to limited interaction between the mixed variables.
MVRSM's surrogate model, based on a linear combination of rectified linear units, avoids all of these problems by having a fixed number of basis functions that contain interaction between all variables, while also having the guarantee that any local optimum is located in points where the integer constraints are satisfied.
These properties cause MVRSM to give competitive performance on two real-life benchmarks, which we have shown experimentally.
It also makes MVRSM more accurate than the state-of-the-art on large-scale synthetic problems (e.g. $>\!50$ variables) and more efficient 
than most competitors.
All of this is achieved using the same hyperparameter settings for MVRSM, while for other methods it might be necessary to spend some time on finding the right settings.


For future work we will investigate the exploration part of the surrogate model, for example by applying techniques with more theoretical guarantees such as Thompson sampling, and 
adapt the method to efficiently deal with categorical and conditional variables and with constraints.


\section*{Acknowledgements}

This work is part of the research programme Real-time data-driven maintenance logistics with project number 628.009.012, which is financed by the Dutch Research Council (NWO).
The authors thank 
Erik Daxberger
for providing the code for generating one of \mbox{MiVaBO's} synthetic test functions (called MiVaBO synthetic function in this paper),
Frederik Rehbach for providing information on the ESP problem, and anonymous reviewers of an earlier version of this paper for providing constructive feedback.

\bibliographystyle{abbrv}
\bibliography{mybib}

\appendix

\section{Details for generating mixed basis functions}\label{app:pb}

In this section we show how to choose $p_\omega$ and $p_b$ from Definition~\ref{def:mixed} in such a way that the mixed $z$-functions are never completely outside the domain $\Xc\times\Xd$.
We recommend to choose $p_\omega$ to be a uniform distribution over $[-\frac{1}{\dc+\dd}, \frac{1}{\dc+\dd}]^{\dc+\dd}$.
This way, the term $\wc_\k^T \xc + \wdd_\k^T \xd$ will not take on large values, which might cause numerical problems.

After sampling  $\omega_\k = \left[\begin{array}{l}\wc_\k \\ \wdd_\k \end{array}\right]$ from $p_\omega$, we look for two cornerpoints $\mathbf q_1, \mathbf q_2$ of the space $\Xc\times\Xd$.
For every dimension $i$, the $i$-th element of corner points $\mathbf q_1, \mathbf q_2$ is determined by
\begin{align}
    {q_1}_i = \left\{\begin{array}{lr}l_i, & \ {\omega_\k}_i\geq 0,\\ u_i, &\  {\omega_\k}_i< 0,\end{array}\right.\\
    {q_2}_i = \left\{\begin{array}{lr}u_i, & \ {\omega_\k}_i\geq 0,\\ l_i, & \  {\omega_\k}_i< 0.\end{array}\right.
\end{align}
Here, $l_i$ and $u_i$ are the lower and upper bounds of the $i$-th variable respectively, so this gives 
\begin{align}
    \omega_\k^T \mathbf q_1 \leq \wc_\k^T \xc + \wdd_\k^T \xd \leq \omega_\k^T \mathbf q_2 \ \forall \  \xc\in \Xc, \xd \in \Xd. \label{eq:boundvw}
\end{align}
Now we calculate the distance from the hyperplane generated by $\omega_\k$ to these corner points, which can be done with the inner product:
\begin{align}
    \beta_1 & = \omega_\k^T \mathbf q_1, \ 
    \beta_2  = \omega_\k^T \mathbf q_2. \label{eq:beta}
\end{align}
By the way $\beta_1$ and $\beta_2$ are constructed and because $l_i<u_i$, we now have $\beta_1<\beta_2$.
We choose $p_b$ equal to the uniform distribution over $[-\beta_2, -\beta_1]$.

Next we prove that this choice of $p_b$ prevents the hyperplane $z_\k(\xc,\xd)=0$ from being completely outside the set $\Xc\times\Xd$.

\begin{theorem}
    Let $\omega_\k = \left[\begin{array}{l}\wc_\k \\ \wdd_\k \end{array}\right]$ be sampled from any continuous probability distribution $p_\omega$ and let $b_\k$ be sampled from the uniform distribution over $[-\beta_2, -\beta_1]$, with $\beta_1$, $\beta_2$ as in~\eqref{eq:beta}.
    Let $z_\k(\xc,\xd)=\wc_\k^T \xc + \wdd_\k^T\xd + b_\k$.
    Then, there exists a $(\xc, \xd) \in \Xc\times \Xd$ such that $z_\k(\xc,\xd)=0$.
\end{theorem}
\begin{proof}

Suppose that $(\xc, \xd)\not \in \Xc\times \Xd$ for all $(\xc,\xd)$ for which $z_\k(\xc,\xd)=0$.
Then from~\eqref{eq:boundvw}, at least one of the following inequalities holds:
\begin{align}
    \wc_\k^T \xc + \wdd_\k^T \xd & > \omega_\k^T\mathbf q_2,\label{eq:qq2}\\
    \wc_\k^T \xc + \wdd_\k^T \xd & < \omega_\k^T\mathbf q_1.\label{eq:qq1}
\end{align}
Because $z_\k(\xc,\xd)=0$,
we have $b_\k = -\wc_\k^T \xc - \wdd_\k^T\xd$.
Because $b_\k$ is sampled from $p_b$, from~\eqref{eq:beta} we also have $-\omega_\k^T\mathbf q_2 \leq b_\k \leq -\omega_\k^T\mathbf q_1$.
This gives $\omega_\k^T\mathbf q_1 \leq  \wc_\k^T \xc + \wdd_\k^T \xd \leq \omega_\k^T\mathbf q_2$, which is in conflict with~\eqref{eq:qq2}-\eqref{eq:qq1}.
By contradiction, there has to exist a $(\xc,\xd)\in \Xc\times\Xd$ with $z_\k(\xc,\xd)=0$. 
\end{proof}




\section{Details on the exploration step for integer variables}\label{app:expl}

This section gives more details on the last step of the MVRSM algorithm, the exploration step.
For the integer variables $\xd^*$, the exploration step consists of determining a random perturbation $\delta_d\in \Z^{\dd}$ that is added to the solution. 
Our approach is similar to the one in~\cite[Sec. 3.4]{bliek2019black}, except that we allow perturbations that are larger than $1$.
We determine $\delta_d$ according to Algorithm~\ref{alg:deltad}.
\begin{algorithm}[tb]
\caption{Determining $\delta_d$}\label{alg:deltad}
 \begin{algorithmic}
 \Require Domain $\Xd$, current solution $\xd^*$
 \Ensure $\delta_d \in \Z^{\dd}$
 \For{$i=1, \ldots, \dd$}
    \State $r_1 \sim \mathrm{Uniform}[0,1]$
    \State $r_2 \sim \mathrm{Uniform}[0,1]$ \Comment{Whether to increase or decrease $x_i$, the $i$-th element of $\xd^*$}
    \State $p = 1/(\dc+\dd)$
    \While{$r_1 < p$}
    \If{$x_i=l_i$}
        $x_i\leftarrow x_i+1$
    \ElsIf{$x_i=u_i$}
        $x_i\leftarrow x_i-1$
    \Else
        \If{$r_2<0.5$}
            $x_i\leftarrow x_i+1$
        \Else
            \ $x_i\leftarrow x_i-1$
        \EndIf
    \EndIf
    \State $r_1 \leftarrow 2r_1$
 \EndWhile
 \EndFor
\end{algorithmic}
\end{algorithm}

For the continuous variables, we use the procedure from~\cite{DONEpaper}, adding a random variable $\delta_c\in \R^{\dc}$ to $\xc^*$. For each continuous variable $\xc[i]$, $\delta_c$ is zero-mean normally distributed with a standard deviation of $0.1|\Xc[i]|/\sqrt{\dc+\dd}$.
The exploration step for both integer and continuous variables is done in such a way that the solution stays within the bounds $\Xc,\Xd$.

\section{
Additional experiments on synthetic benchmark functions}

In this section we 
show the results on some additional synthetic benchmarks with lower dimensions.

\subsubsection*{Func3C}
This benchmark was taken from~\cite[Sec. 5.1]{ru2019bayesian}. It has $3$ categorical and $2$ continuous variables.

\begin{figure}[tbp]
\centering
\includegraphics[width=0.99\columnwidth]{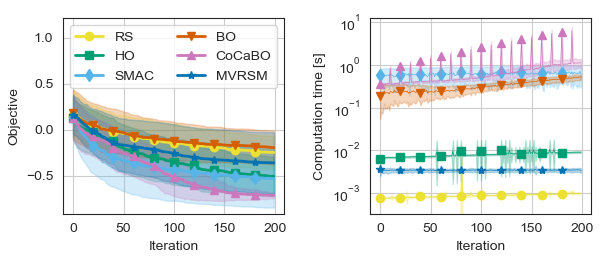}
\caption{Results on the func3C~\cite[Sec. 5.1]{ru2019bayesian} benchmark
($3$ categorical, $2$ continuous),
averaged over $100$ runs. The compared methods are random search (RS), HyperOpt (HO), SMAC, Bayesian optimisation (BO),  CoCaBO and MVRSM.
} \label{fig:func3C}
\end{figure}

Figure~\ref{fig:func3C} shows the results of $200$ iterations averaged over $100$ runs.
We have managed to reproduce the results from~\cite[Fig. 6(b)]{ru2019bayesian} for both HO (also called TPE) and CoCaBO.
Our result of SMAC is better here due to not using the default setting.
As this benchmark has categorical variables and was one of CoCaBO's benchmarks, we expect CoCaBO to perform best, which it does, though it uses more computation time than the other methods.

\subsubsection*{Rosenbrock10}
The Rosenbrock function%
\footnote{\label{fn:surjano}Details available at \url{https://www.sfu.ca/~ssurjano/optimization.html}
} %
is a standard benchmark in continuous optimisation that can be scaled to any dimension. 
For any dimension, the function has its global minimum in the point $(1,1,1,\ldots, 1)$, where it achieves the value $0$.
This benchmark has a dimension of $10$, but $3$ of the variables were adapted to integers in $\Xd = \{-2,-1,0,1,2\}^3$.
The $7$ remaining continuous variables were limited to $\Xc = [-2,2]^7$.
The function was scaled with a factor $1/300$, and uniform noise in $[0,10^{-6}]$ was added to every function evaluation.
This problem is of the same scale as the problem of gradient boosting hyperparameter tuning~\cite[Sec. 4(a)]{daxberger2019mixed}.

\begin{figure}[tbp]
\centering
\includegraphics[width=0.99\columnwidth]{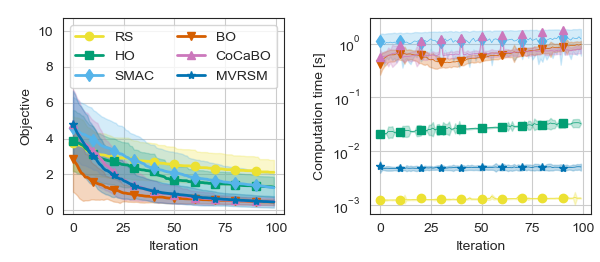}
\caption{Results on the Rosenbrock10 benchmark
($3$ integer, $7$ continuous), 
averaged over $100$ runs. 
This problem is of a similar scale as gradient boosting hyperparameter tuning~\cite[Sec. 4(a)]{daxberger2019mixed}.
} \label{fig:R10}
\end{figure}

Figure~\ref{fig:R10} shows the results of $100$ iterations averaged over $100$ runs.
Surprisingly, BO has the best performance, though it is much slower than MVRSM.
This method is typically used on continuous problems and widely assumed to be inadequate for discrete or mixed problems.
Here, we have experimentally shown that this is a false assumption.
MVRSM and CoCaBO get similar results as BO on this problem, with MVRSM being the most efficient.

\subsubsection*{MiVaBO synthetic function}

We also compare with one of the randomly generated synthetic test functions from~\cite[Appendix C.1, Gaussian weights variant]{daxberger2019mixed} .
This problem has $16$ variables of which $8$ integer and $8$ continuous.
No bounds were reported so we set them to $\Xd=\{0,1,2,3\}^8$ for the integer variables and $\Xc=[0,3]^8$ for the continuous variables.
We generated $8$ of these random functions and ran all algorithms $16$ times on each of them for $100$ iterations.

\begin{figure}[tbp]
\includegraphics[width=0.99\columnwidth]{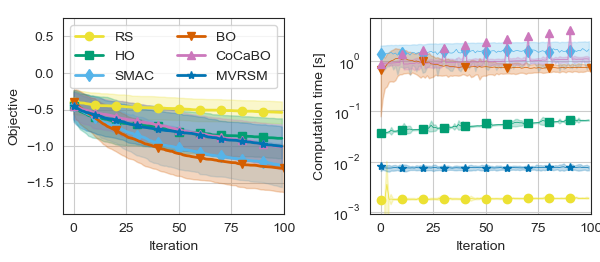}
\caption{Results on $8$ randomly generated MiVaBO synthetic benchmarks~\cite[Appendix C.1, Gaussian weights variant]{daxberger2019mixed} 
($8$ integer, $8$ continuous),
averaged over $16$ runs and over the $8$ different benchmarks.
} \label{fig:lm}
\end{figure}

Figure~\ref{fig:lm} shows the average over all $128$ runs.
Again, the standard BO algorithm performs best, which is a result that was not concluded in~\cite{daxberger2019mixed}.

\end{document}